\documentclass[11pt]{book}
\usepackage{arydshln}
\usepackage{Wiley-AuthoringTemplate}
\usepackage{chapterbib}
\usepackage[sectionbib,authoryear]{natbib}

\makeindex

\begin{document}

\frontmatter

\mainmatter
\chapter[This is Chapter One Title containing authors and affiliations]{Manipulating Reinforcement Learning: Stealthy Attacks on Cost Signals\protect}

\author*[1]{Yunhan Huang}
\author[1]{Quanyan Zhu}

\address[1]{\orgdiv{Department
of Electrical and Computer Engineering}, 
\orgname{New York University}, 
\postcode{11201}, \countrypart{370 Jay Street}, 
     \city{Brooklyn}, \street{New York}, \country{USA}}%



\address*{Corresponding Author: Yunhan Huang; \email{yh.huang@nyu.edu}}

\maketitle

\begin{abstract}{Abstract}
This chapter studies emerging cyber attacks on reinforcement learning (RL) and introduces a quantitative approach to analyze the vulnerabilities of RL. Focusing on adversarial manipulation on the cost signals, we analyze the performance degradation of Temporal Difference (TD) learning and $Q$-learning algorithms under the manipulation. For TD learning, the approximation learned from the manipulated costs has an approximation error bounded by a constant times the magnitude of the attack. The effect of the adversarial attacks on the bound does not depend on the choice of $\lambda$, the weighting factor. In $Q$-learning, we show that $Q$-learning algorithms converge under stealthy attacks and bounded falsifications on cost signals. We characterize the relation between the falsified cost and the $Q$-factors as well as the policy learned by the learning agent which provides fundamental limits for feasible offensive and defensive moves. We propose a robust region in terms of the cost within which the adversary can never achieve the targeted policy. We provide conditions on the falsified cost which can mislead the agent to learn an adversary's favored policy. A case study of TD learning is provided to corroborate the results. 
\end{abstract}


Reinforcement learning (RL) is a powerful paradigm for online decision-making in unknown environment. Recently, many advanced RL algorithms have been developed and applied to various scenarios including video games (e.g., \cite{mnih2015human}), transportation (e.g., \cite{arel2010reinforcement}), network security (e.g., \cite{huang2019adaptive,zhu2009dynamic}), robotics (e.g., \cite{kober2013reinforcement}) and critical infrastructures (e.g., \cite{ernst2004power}). However, the implementation of RL techniques requires accurate and consistent feedback from environment. It is straightforward to fulfill this requirement in simulation while in practice, accurate and consistent feedback from the environment is not guaranteed, especially in the presence of adversarial interventions. For example, adversaries can manipulate cost signals
by performing data injection attack and prevent an agent from receiving cost signals by jamming the communication channel. With inconsistent and/or manipulated feedback from environment, the RL algorithm can either fail to learn a policy or misled to a pernicious policy. The failure of RL algorithms under adversarial intervention can lead to a catastrophe to the system where the RL algorithm has been applied. For example, self-driving platooning vehicles can collide with each other when their observation data are manipulated (see \cite{behzadan2019adversarial}); drones equipped with RL techniques can be weaponized by terrorists to create chaotic and
vicious situations where they are commanded to collide to a crowd or a building (see \cite{huang2019deceptive,xu2015cyber}).

Hence, it is imperative to study RL with maliciously intermittent or manipulated feedback under adversarial intervention. First, it is important to understand the adversarial behaviors of the attacker. To do so, one has to establish a framework that characterizes the objective of the attacker, the actions available to the attacker and the information at his disposal. Secondly, it is also crucial to understand the impacts of the attacks on RL algorithms. The problems include how to measure the impacts, how to analyze the behavior of the RL algorithms under different types of attacks and how to mathematically and/or numerically compute the results of RL algorithms under attack. With the understanding of the adversarial behavior and the impacts of the adversarial behavior on RL algorithms, the third is to design defense mechanisms that can protect RL algorithms from being degenerated. This could be done by designing robust and secure RL algorithms that can automatically detect and discard corrupted feedback, deploying cryptographic techniques to ensure confidentiality, integrity and building backup communication channels to ensure availability. 

Despite the importance of understanding RL in malicious setting, very few works have studied RL with maliciously manipulated feedback or intermittent feedback. One type of related works studies RL algorithms under corrupted reward signals (see \cite{everitt2017reinforcement,wang2018reinforcement}). In \cite{everitt2017reinforcement}, the authors investigate RL problems where agents receive false rewards from environment. Their results show that reward corruption can impede the performance of agents, and can result in disastrous consequences for highly intelligent agents. Another type of work studies delayed evaluative feedback signals without the presence of malicious adversaries (see \cite{tan2008integrating,watkins1989learning}). The study of
RL under malicious attacks from a security point of view appeared in the recent past (\cite{huang2019deceptive,ma2019policy,lin2017tactics,behzadan2018faults}). In \cite{huang2019deceptive}, the authors study RL under malicious falsification on cost signals and introduces a quantitative framework of attack models to understand the vulnerabilities of RL. Ma et al. focuses on security threats on batch RL and control where attacker aims to poison the learned policy (\cite{ma2019policy}). \cite{lin2017tactics} and \cite{behzadan2018faults} focus on deep RL which involves deep natural networks (DNNs) for function approximation.

In this chapter, we first introduce RL techniques built on a Markov decision process framework and provide self contained background on RL before we discuss security problems of RL. Among RL techniques, of particular interest to us are two frequently used learning algorithms: TD learning and $Q$-learning. We then introduce general security concerns and problems in the field of RL. Security concerns arise from the fact that RL technique requires accurate and consistent feedback from environment, timely deployed controls and reliable agents, which are hard to guarantee under the presence of adversarial attacks. The discussion of general security concerns in this chapter invokes a large number of interesting problems yet to be done. In this chapter, we focus on one particular type of problems where the cost signals that the RL agent receives are falsified or manipulated as a result of adversarial attacks. In this particular type of problems, a general formulation of attack models is discussed by defining the objectives, information structure and the capability of the adversary. We analyze the attacks on both TD learning and $Q$-learning algorithms. We develop important results that tell the fundamental limits of the adversarial attacks. For TD learning, we characterize the bound on the approximation error that can be induced by the adversarial attacks on the cost signals. The choice of $\lambda$ does not impact the bound of the induced approximation error. In the $Q$-learning scenario, we aim to address two fundamental questions. The first is to understand the impact of the falsification of cost signals on the convergence of $Q$-learning algorithm. The second is to understand how the RL algorithm can be misled under the malicious falsifications. This chapter ends with an educational example that explains how the adversarial attacks on cost signals can affect the learning results in TD learning.

The rest of this chapter is organized as follows. Sec. \ref{IntroRL} gives a basic introduction of Markov decision process and RL techniques with a focus on TD($\lambda$) learning and $Q$-learning. In Sec. \ref{GeneralSecurityProblems}, we discuss general security concerns and problems in the field of RL. A particular type of attacks on the cost signals is studied on both TD($\lambda$) learning and $Q$-learning in Sec. \ref{RLManipulatedCost}. Sec. \ref{CaseStudy} comes an educational example that illustrates the adversarial attacks on TD($\lambda$) learning. Conclusions and future works are included in Sec. \ref{conc}.

\section{Introduction of Reinforcement Learning}\label{IntroRL}

Consider an RL agent interacts with an unknown environment and attempts to find the optimal policy minimizing the received cumulative costs. The environment is formalized as a Markov Decision Process (MDP) denoted by $\langle S, A, g, P, \alpha \rangle$. The MDP has a finite state space denoted by $S$. Without loss of generality, we assume that there are $n$ states and $S=\{1,2,\cdots,n\}$. The state transition depends on a control. The control space is also finite and denoted by $A$. When at state $i$, the control must be chosen from a given finite subset of $A$ denoted by $U(i)$. At state $i$, the choice of a control $u\in U(i)$ determines the transition probability $p_{ij}(u)$ to the next state $j$. The state transition information is encoded in $P$. The agent receives a running cost that accumulates additively over time and depends on the states visited and the controls chosen. At the $k$th transition, the agent incurs a cost $\alpha^k g(i,u,j)$, where $g:S\times A\times S \rightarrow \mathbb{R}$ is a given real-valued function that describes the cost associated with the states visited and the control chosen, and $\alpha\in(0,1]$ is a scalar called the discount factor. 

The agent is interested in policies, i.e., sequences $\pi = \{\mu_0,\mu_1,\cdots\}$ where $\mu_k:S\rightarrow A, k=0, 1, \cdots,$ is a function mapping states to controls with $\mu_k(i)\in U(i)$ for all states $i$. Denote $i_k$ the state at time $k$. Once a policy $\pi$ is fixed, the sequence of states $i_k$ becomes a Markov chain with transition probabilities $P(i_{k+1}=j|i_k=i)= p_{ij}(\mu_k(i))$. In this chapter, we consider infinite horizon problems, where the cost accumulates indefinitely. In the infinite horizon problem, the total expected cost starting from an initial state $i$ and using a policy $\pi = \{\mu_0,\mu_1,\cdots\}$ is 
$$
J^\pi(i) = \lim_{N\rightarrow \infty } E\left[ \sum_{k=0}^N \alpha_k g\left(i_k,\mu_k(i_k),i_{k+1}\right) \;\middle|\;i_0= i \right],
$$
the expected value is taken with respect to the probability distribution of the Markov chain $\{i_0,i_1,i_2,\cdots\}$. This distribution depends on the initial state $i_0$ and the policy $\pi$. The optimal cost-to-go starting from state $i$ is denoted by $J^*(i)$; that is,
$$
J^*(i)  = \min_\pi J^\pi (i).
$$
We can view the costs $J^*(i),i=1,2,\cdots$, as the components of a vector $J^*\in \mathbb{R}^n$. Of particular interest in the infinite-horizon problem are stationary policies, which are policies of the form 
$\pi=\{\mu,\mu,\cdots\}$. The corresponding cost-to-go vector is denoted by $J^\mu \in \mathbb{R}^n$.

The optimal infinite-horizon cost-to-go functions $J^*(i),i=1,2,3,\cdots$, also known as value functions, arise as a central component of algorithms as well as performance metrics in many statistics and engineering applications. Computation of the value functions relies on solving a system of equations
\begin{equation}\label{BellmanEqn}
J^*(i) = \min_{u\in U(i)} \sum_{j=1}^n p_{ij}(u)(g(i,u,j)+\alpha J^*(j)),\ \ \ i=1,2,\cdots,n,
\end{equation}
referred to as Bellman's equation (\cite{bertsekas1996neuro,sutton1998introduction}), which will be at the center of analysis and algorithms in RL. If $\mu(i)$ attains the minimum in the right-hand side of Bellman's equation (\ref{BellmanEqn}) for each $i$, then the stationary policy $\mu$ should be optimal (\cite{bertsekas1996neuro}). That is, for each $i\in S$,
$$
\mu^*(i) = \arg\min_{\mu(i)\in U(i)} \sum_{j=1}^n p_{ij}\left(\mu(i)\right) \left(  g(i,\mu(i),j) + \alpha J^*(j) \right).
$$

The efficient computation or approximation of $J^*$ and an optimal policy $\mu^*$ is the major concern of RL. In MDP problems where the system model is known and the state space is reasonably large, value iteration, policy iteration, and linear programming are the general approaches to find the value function and the optimal policy. Readers unfamiliar with these approaches can refer to Chapter 2 of \cite{bertsekas1996neuro}. 

It is well-known that RL refers to a collection of techniques for solving MDP under two practical issues. One is the overwhelming computational requirements of solving Bellman's equations because of a colossal amount of states and controls, which is often referred to as Bellman's ``curse of dimensionality''. In such situations, an approximation method is necessary to obtain sub-optimal solutions. In approximation methods, we replace the value function $J^*$ with a suitable approximation $\tilde{J}(r)$, where $r$ is a vector of parameters which has much lower dimension than $J^*$. There are two main function approximation architectures: linear and nonlinear approximations. The approximation architecture is linear if $\tilde{J}(r)$ is linear in $r$. Otherwise, the approximation architecture is nonlinear. Frequently used nonlinear approximation methods include polynomials based approximation, wavelet based approximation, and approximation using neural network. The topic of deep reinforcement learning studies the cases where approximation $\tilde{J}(r)$ is represented by a deep neural network (\cite{mnih2015human}).

Another issue comes from the unavailability of the environment dynamics; i.e., the  transition probability is either unknown or too complex to be kept in memory. In this circumstance, one alternative is to simulate the system and the cost structure. With given state space $S$ and the control space $A$, a simulator or a computer that generates a state trajectory using the probabilistic transition from any given state $i$ to a generated successor state $j$ for a given control $u$. This transition accords with the transition probabilities $p_{ij}(u)$, which is not necessarily known to the simulator or the computer. Another alternative is to attain state trajectories and corresponding costs through experiments. Both methods allow the learning agent to observe their own behavior to learn how to make good decisions.  
It is clearly feasible to use repeated simulations to find the approximate of the transition model of the system $P$ and the cost functions $g$ by averaging the observed costs. This approach is usually referred to as model-based RL. As an alternative, in model-free RL, transition probabilities are not explicitly estimated, but instead the value function or the approximated value function of a given policy is progressively calculated by generating several sample system trajectories and associated costs. Of particular interest to us in this chapter is the security of model-free RL. This is because firstly, model-free RL approaches are the most widely applicable and practical approaches that have been extensively investigated and implemented; secondly, model-free RL approaches receive observations or data from environment successively and consistently. This makes model-free RL approaches vulnerable to attacks. An attack can induce an accumulative impact on succeeding learning process. The most frequently used algorithms in RL are TD learning algorithms and $Q$-learning algorithms. Hence, in this chapter, we will focus on the security problems of these two learning algorithms as well as their approximate counterparts. 

\subsection{TD Learning}
Temporal difference (TD) learning is an implementation of the Monte Carlo policy evaluation algorithm that incrementally updates the cost-to-go estimates of a given policy $\mu$, which is an important sub-class of general RL methods. TD learning algorithms, introduced in many references, including \cite{bertsekas1996neuro}, \cite{sutton1998introduction} and \cite{tsitsiklis1997analysis}, generates an infinitely long trajectory of the Markov Chain $\{i_0,i_1,i_2,\cdots\}$ from simulator or experiments by fixing a policy $\mu$, and at time $t$ iteratively updates the current estimate $J^\mu_t$ of $J^\mu$ using an iteration that depends on a fixed scalar $\lambda\in [0,1]$, and on the temporal difference 
$$
d_t(i_k,i_{k+1}) = g(i_k,i_{k+1}) + \alpha  J^\mu_t (i_{k+1}) - J^\mu_t(i_k), \forall t =0,1,\cdots, \forall k\leq t.
$$

The incremental updates of TD($\lambda$) have many variants. In the most straightforward implementation of TD($\lambda$), all of the updates are carried out simultaneously after the entire trajectory has been simulated. This is called the off-line version of the algorithm. On the contrary, in the on-line implementation of the algorithm, the estimates update once at each transition. Under our discount MDP, a trajectory may never end. If we use an off-line variant of TD($\lambda$), we may have to wait infinitely long before a complete trajectory is obtained. Hence, in this chapter, we focus on an on-line variant. The update equation for this case becomes 
\begin{equation}\label{TDLearningUpdates}
J^\mu_{t+1}(i) = J^\mu_t(i) + \gamma_t(i) z_t(i) d_t(i),\ \ \ \forall i,
\end{equation}
where the $\gamma_t(i)$ are non-negative stepsize coefficients and $z_t(i)$ is the eligibility coefficients defined as
\begin{equation*}
z_t(i) = \begin{cases}
\alpha \lambda z_{t-1}(i),\ \ \ \textrm{if }i_t\neq i,\\
\alpha \lambda z_{t-1}(i)+1,\ \ \ \textrm{if }i_t = i.\\
\end{cases}
\end{equation*}
This definition of eligibility coefficients gives the every-visit TD($\lambda$) method. In every-visit TD($\lambda$) method, if a state is visited more than once by the same trajectory, the update should also be carried out more than once.

Under a very large number of states or controls, we have to resort to approximation methods. Here, we introduce TD($\lambda$) with linear approximation architectures. We consider a linear parametrization of the form: 
$$
\tilde{J}(i,r) = \sum_{k=1}^K r(k) \phi_k (i).
$$
Here, $r= (r(1),\dots,r(K))$ is a vector of tunable parameters and $\phi_k(\cdot)$ are fixed scalar functions defined on the state space. The form can be written in a compact form:
$$
\tilde{J}(r) = (\tilde{J}(1,r),\dots,\tilde{J}(n,r))= \Phi r,
$$
where
$$
\Phi = \begin{bmatrix}
| & \ & | \\
\phi_1 & \cdots  & \phi_K \\
| & \ & |\\
\end{bmatrix}= 
\begin{bmatrix}
\textrm{---} & \phi'(1) & \textrm{---} \\
\cdots & \cdots & \cdots \\
\textrm{---} & \phi'(n) & \textrm{---}\\ 
\end{bmatrix},
$$
with $\phi_k= (\phi_k(1),\dots,\phi_k(n))$ and $\phi(i)=(\phi_1(i),\dots,\phi_K(i))$. We assume that $\Phi$ has linearly independent columns. Otherwise, some components of $r$ would be redundant.

Let $\eta_t$ be the eligibility vector for the approximated TD($\lambda$) problem which is of dimension $K$. With this notation, the approximated TD($\lambda$) updates are given by
\begin{equation}\label{ApproximateTDupdate}
r_{t+1} = r_t + \gamma_t d_t \eta_t,
\end{equation}
where
\begin{equation}\label{ApproximateEligibility}
\eta_{t+1} = \alpha \lambda \eta_t + \phi(i_{t+1}).
\end{equation}
Here, $d_t = g(i_t,i_{t+1}) + \alpha r_{t}' \phi(i_{t+1}) - r_t'\phi(i_t)$.

The almost-sure convergence of $r_t$ generated by (\ref{ApproximateTDupdate}) and (\ref{ApproximateEligibility}) is guaranteed if the conditions in Assumption 6.1 in \cite{bertsekas1996neuro} hold. It will converge to the solution of 
$$
Ar +b =0,
$$
where $A=\Phi ' D (M-I)\Phi$ and $b= \Phi'D q$. Here, $D$ is a diagonal matrix with diagonal entries $d(1),d(2),\cdots,d(n)$, and $d(i)$ is the steady-state probability of state $i$; the matrix $M$ is given by $M=(1-\lambda)\sum_{m=0}^\infty \lambda^m (\alpha P_{\mu})^{m+1}$ and the vector $b$ is given by $b=\Phi'D q$ with $q=\sum_{m=0}^\infty (\alpha \lambda P_\mu)^m \bar{g}$, where $\bar{g}$ is a vector in $\mathbb{R}^n$ whose $i$th component is given by $\bar{g}(i)=\sum_{j=1}^n p_{ij}(\mu(i))g(i,\mu(i),j)$. The matrix $P_\mu$ is the transition matrix defined by $P_\mu\coloneqq [P_{\mu}]_{i,j} = p_{ij}(\mu(i))$. A detailed proof of the convergence is provided in \cite{tsitsiklis1997analysis} and \cite{bertsekas1996neuro}.

Indeed, TD($\lambda$) with linear approximations is a much more general framework. The convergence of $J^\mu_t$ in TD($\lambda$) without approximation follows immediately if we let $K=n$ and $\Phi=I_n$ where $I_n$ is $n\times n$ identity matrix. 

\subsection{Q-Learning}

$Q$-learning method is initially proposed in \cite{watkins1992q} which updates estimates of the $Q$-factors associated with an optimal policy. $Q$-learning is proven to be an efficient computational method that can be used whenever there is no explicit model of the system and the cost structure. First, we introduce the first notion of the $Q$-factor of a state-control pair $(i,u)$, defined as
\begin{equation}\label{DefineQfactor}
Q(i,u) = \sum_{j=0}^n p_{ij}(u)(g(i,u,j)+ \alpha J(j)).
\end{equation}
The optimal $Q$-factor $Q^*(i,u)$ corresponding to a pair $(i,u)$ is defined by (\ref{DefineQfactor}) with $J(j)$ replaced by $J^*(j)$. It follows immediately from Bellman's equation that
\begin{equation}\label{QBellmanEquation}
Q^*(i,u) = \sum_{j=0}^n p_{ij}(u) \left( g(i,u,j) + \alpha \min_{v\in U(j)} Q^*(j,v)\right).
\end{equation}
Indeed, the optimal $Q$-factors $Q^*(i,u)$ are the unique solution of the above system by Banach fixed-point theorem (see \cite{kreyszig1978introductory}). 

Basically, $Q$-learning computes the optimal $Q$-factors based on samples in the absence of system model and cost structure. It updates the $Q$-factors following
\begin{equation}\label{QlearningUpdate}
Q_{t+1}(i,u) = (1-\gamma_t)Q_t(i,u) + \gamma_t  \left( g(i,u,\bar{i})  + \alpha \min_{v\in U(\bar{i})} Q_t(\bar{i},v)  \right),
\end{equation}
where the successor state $\bar{i}$ and $g(i,u,\bar{i})$ is generated from the pair $(i,u)$ by simulation or experiments, i.e., according to the transition probabilities $p_{i\bar{i}}(u)$. For more intuition and interpretation about $Q$-learning algorithm, one can refer to Chapter 5 of \cite{bertsekas1996neuro}, \cite{sutton1998introduction} and \cite{watkins1992q}.

If we assume that stepsize $\gamma_t$ satisfies $\sum_{t=0}^\infty \gamma_t = \infty$ and $\sum_{t=0}^\infty \gamma_t^2 < \infty$, we obtain the convergence of $Q_t(i,u)$ generated by (\ref{QlearningUpdate}) to the optimal $Q$-factors $Q^*(i,u)$. A detailed proof of convergence in provided in  \cite{borkar2000ode} and Chapter 5 of \cite{bertsekas1996neuro}.

\section{Security Problems of Reinforcement Learning}\label{GeneralSecurityProblems}

Understanding adversarial attacks on RL systems is essential to develop effective defense mechanisms and an important step toward trustworthy and safe RL. The reliable implementation of RL techniques usually requires accurate and consistent feedback from the environment, precisely and timely deployed controls to the environment and reliable agents (in multi-agent RL cases). Lacking any one of the three factors will render failure to learn optimal decisions. These factors can be used by adversaries as gateways to penetrate RL systems. It is hence of paramount importance to understand and predict general adversarial attacks on RL systems targeted at the three doorways. In these attacks, the miscreants know that they are targeting RL systems, and therefore, they tailor their attack strategy to mislead the learning agent. Hence, it is natural to start with understanding the parts of RL that adversaries can target at.

\begin{figure}
    \centering
    \includegraphics[width=0.9\linewidth]{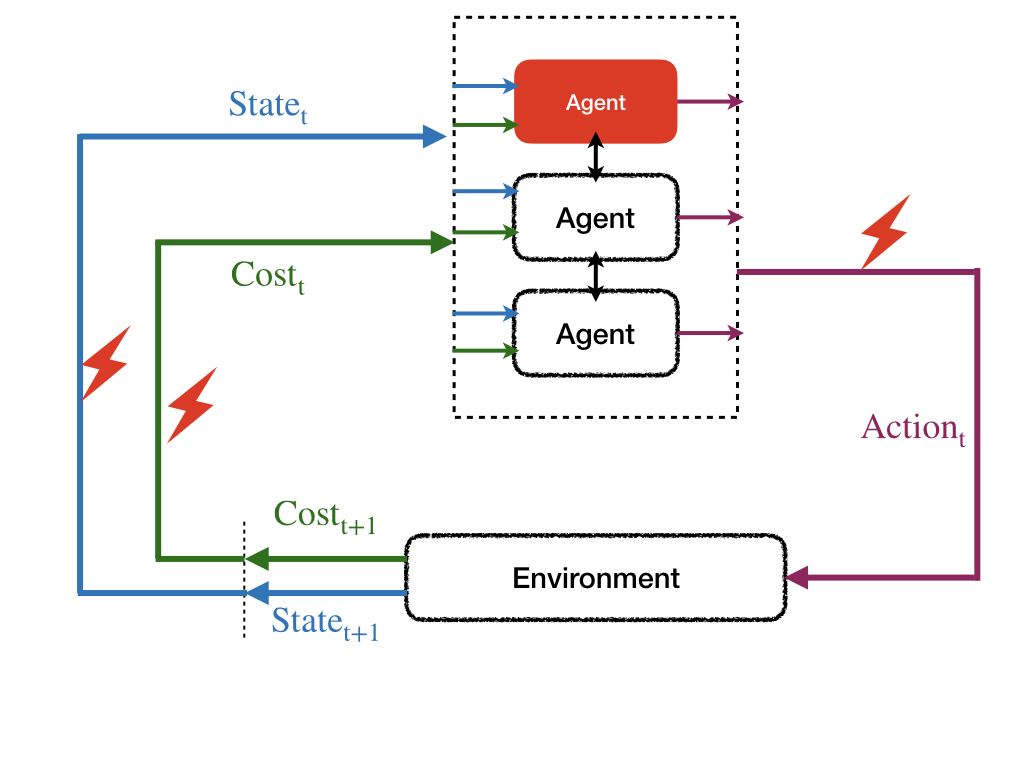}
    \caption{Attacks on RL systems: Adversaries can manipulate actions, costs, and state signals in the feedback system. Agents can also be compromised in the learning process.}
    \label{fig:GeneralAttack}
\end{figure}

Fig. \ref{fig:GeneralAttack} illustrates different types of attacks on the RL system. One type of attack aims at the state which is referred to as state attacks. Attacks on state signals can happen if the remote sensors in the environment are compromised or the communication channel between the agent and the sensors is jammed or corrupted. In such circumstances, the learning agent may receive a false state observation $\tilde{i}_k$ of the actual state $i_k$ at time $k$ and/or may receive a delayed observation of the actual state or even never receive any information regarding the state at time $k$. An example of effortless state attacks is sequential blinding/blurring of the cameras in a deep RL-based autonomous vehicle via lasers/dirts on lens, which can lead to learning false policies and hence lead to catastrophic consequences. Based on its impact on the RL systems, state attacks can be classified into two groups: i) Denial of Service (DoS); and ii) integrity attacks. The main purpose of DoS attacks is to deny access to sensor information. Integrity attacks are characterized by the modification of sensor information, compromising their integrity.

Another type of attacks targets at cost signals. In this type of attacks, the adversary aims to corrupt the cost signals that the learning agent has received with a malicious purpose of misleading the agent. Instead of receiving the actual cost signal $g_k = g(i_k,u_k,j_{k+1})$ at time $k$, the learning agent receives a manipulated or falsified cost signal $\tilde{g}_k$. The corruption of cost signals comes from the false observation of the state in cases where the cost is predetermined by the learning agent. In other cases where cost signals are provided directly by the environment or a remote supervisor, the cost signal can be corrupted independently from the observation of the state. The learning agent receives falsified cost signals even when the observation of the state is accurate. In the example of autonomous vehicle, if the cost depends on the distance of the deep RL agent to a destination as measure by GPS coordinates, spoofing of GPS signals by the adversary may result in incorrect reward signals, which can translate to incorrect navigation policies (See \cite{behzadan2018faults}). 

The corruption of cost signals may lead to a doomed policies. For example, \cite{clarkfaulty} has trained an RL agent on a boat racing game. High observed reward misleads the agent to go repeatedly in a circle in a small lagoon and hit the same targets, while losing every race. Or even worse, the learning agent may be misled to  lethal policies that would result in self-destruction.

The learning agent influences the environment by performing control $u$ via actuators. There is a type of attacks targeting the actuators. If the actuator is exposed to the adversary, the control performed will be different than the one determined by the agent. Hence, during the learning process, the agent will learn a deteriorated transition kernel and a matched reward function. An example of this type of attack is RL-based unmanned aerial vehicles may receive corrupted control signals which would cause malfunctioning of its servo actuators.

To understand an attack model, one needs to specify three ingredients: objective of an adversary, actions available to the adversary, and information at his disposal. The objective of an adversary describes the goal the adversary aims to achieve out of the attacks. Objectives of an adversary include but not limit to maximizing the agent's accumulative costs, misleading the agent to learn certain policies, or diverging the agent's learning algorithm. To achieve the objective, adversaries need a well-crafted strategy. An attack strategy is developed based on the information and the actions at his disposal. An attack strategy is a map from the information space to the action space. The information structure of the adversary describes the knowledge the adversaries have during the learning and the attack process. Whether the adversaries know which learning algorithms the agent implements, whether the adversaries know the actual state trajectory, controls and costs are decisive for the adversaries to choose their strategies. The scope of the attack strategy also depends on the actions at the adversary's disposal. Due to adversaries' capabilities, adversaries can only launch certain attacks. For example, some adversaries can only manipulate the costs induced at certain states or add white noise to the observations. 

In next section, we focus on cases where adversaries target at the cost signals received by the agent. We analyze the performance of TD learning and $Q$-learning under manipulated costs, respectively, and provide fundamental limits under the attacks on cost signals.

\section{Reinforcement Learning with Manipulated Cost Signals}\label{RLManipulatedCost}
 
Under malicious attacks on cost signals as we have discussed in Sec. \ref{GeneralSecurityProblems}, the RL agent will fail to observe the actual cost feedback from the environment. Instead, the agent receives a cost signal that might be falsified by the adversary. Consider the following MDP with falsified cost denoted by $\langle S,A, g,\tilde{g},P,\alpha \rangle$. Under the falsification, the agent, instead of receiving the actual cost signal $g_t\in \mathbb{R}$ at the $t$-th update, observes a falsified cost signal denoted by $\tilde{g}_t\in\mathbb{R}$. The remaining aspects of the MDP framework stay the same. The adversary's task here is to design falsified cost signals $\tilde{g}$ based on his information structure and the actions available to him so that he can achieve certain objectives.

Suppose the adversary is an insider and he knows what the agent knows at time $t$. In other words, at time $t$, the adversary knows the state trajectory, the control trajectory and the cost signals up to time $t$. The adversary may or may not know the system model. Suppose that the adversary falsifies the cost signals in a stealthy way. In this case, the cost associated with each state-control-state triple is consistently falsified in the same way. The following definition gives a formal definition of the attack.

\begin{definition}
\textbf{(Stealthy Attacks)} If $\tilde{g}_t$ takes the same value for the same state-control-state triple for all $t$; i.e., for $t\neq \tau$, we have $\tilde{g}_t=\tilde{g}_\tau$ if $(i_t,u_t,j_{t+1})=(i_\tau,u_\tau,j_{\tau+1})$, then we say the attacks on the cost signals are stealthy.
\end{definition}

Under stealthy attacks, the falsified cost signals can be given by a function $\tilde{g}:S\times A \times S\rightarrow \mathbb{R}$. Then, at time $t$, the falsified cost the agent receives is $\tilde{g}_t= \tilde{g}(i_t,u_t,i_{t+1})$. Since the transition from $i_{t}$ to $i_{t+1}$ depends on the transition probability, without loss of generality, we consider only the cost structure defined on the state-control pair, i.e., $g(i,u)$ for $i\in S, u\in A$.  

\subsection{TD Learning with Manipulated Cost Signals}

If the RL agent updates the estimates of the cost-to-go function of a given policy $\mu$ according to (\ref{ApproximateTDupdate}) and (\ref{ApproximateEligibility}), then under the stealthy attacks, the algorithm can be written as 
\begin{equation}\label{FalfisiedApproTDUpdate}
\begin{aligned}
\tilde{r}_{t+1} &= \tilde{r}_t + \gamma_t \tilde{d}_t \eta_t,\\
\eta_{t+1} &= \alpha \lambda \eta_t + \phi(i_{t+1}),
\end{aligned}
\end{equation}
where $\tilde{d}_t = \tilde{g}(i_{t},u_t) + \alpha \tilde{r}_{t}' \phi(i_{t+1}) - \tilde{r}_t'\phi(i_t)$. 

Suppose the sequence of parameters $\{\tilde{r}_t\}$ generated by (\ref{FalfisiedApproTDUpdate}) under the falsified cost signals is convergent and converges to $\tilde{r}^*$. (We will show the conditions under which the convergence of $\{\tilde{r}_t\}$ is guaranteed.) Let $r^*$ be the solution of $Ar+b$. In TD($\lambda$), the agent aims to estimate the cost-to-go function $J^\mu$ of a given policy $\mu$. In approximated TD($\lambda$) with linear approximation architecture, $\tilde{J}(i,r)= \phi'(i)r$ serves as an approximation of $J^\mu(i)$ for $i\in S$. One objective of the adversary can be to deteriorate the approximation and estimation of $J^\mu$ by manipulating the costs. One way to achieve the objective is to let $\tilde{r}_t$ generated by (\ref{FalfisiedApproTDUpdate}) converge to $\tilde{r}^*$ such that $\Phi'\tilde{r}^*$ is as a worse approximate of $J^\mu$ as possible.

\begin{lemma}\label{ConvergenceFalTD}
If the sequence of parameters $\{\tilde{r}_t\}$ is generated by the TD($\lambda$) learning algorithm (\ref{FalfisiedApproTDUpdate}) with stealthy and bounded attacks on the cost signals, then the sequence $\{\tilde{r}_t\}$ converges to $\tilde{r}^*$ and $\tilde{r}^*$ is a unique solution of $A r +\tilde{b}=0$, where $\tilde{b}= \Phi' D \sum_{m=0}^\infty (\alpha \lambda P_\mu)^m \tilde{g}$ and $\tilde{g}$ is vector whose $i$-th component is $\tilde{g}(i,\mu(i))$.
\end{lemma}

The proof of Lemma \ref{ConvergenceFalTD} follows the proof of Proposition 6.4 in \cite{bertsekas1996neuro} with $g(i,\mu(i),j)$ replaced by $\tilde{g}(i,\mu(i))$. If the adversary performs stealthy and bounded attacks, he can mislead the agent to learn an approximation $\Phi' \tilde{r}^*$ of $J^\mu$. The distance between $\Phi' \tilde{r}^*$ and $J^\mu$ with respect a norm $\Vert \cdot \Vert$ is what the adversary aims to maximize. The following lemma provides an upper bound of the distance between $\Phi' \tilde{r}^*$ and $J^\mu$.

\begin{lemma}\label{BoundFalApproximationError}
Suppose that $\tilde{r}^*$ is the parameters learned from the manipulated TD($\lambda$) (\ref{FalfisiedApproTDUpdate}). Then, the approximation error under the manipulated satisfies 
\begin{equation}\label{ApproErrorBound}
\Vert \Phi' \tilde{r}^* - J^ \mu \Vert_D \leq \Vert \Phi'\tilde{r}^* - \Phi' r^* \Vert_D + \frac{1-\lambda \alpha}{\sqrt{(1-\alpha)(1+\alpha -2\lambda \alpha)}} \Vert (\Pi-I) J^\mu\Vert_D,
\end{equation}
where $\Vert\cdot\Vert_D$ is the weighted quadratic norm defined by $\Vert J \Vert_D^2 = J' D J = \sum_{i=1}^n d(i) J(i)^2$ and $\Pi = \Phi (\Phi' D \Phi)^{-1} \Phi' D$.
\end{lemma}
\begin{proof}
A direct application of triangle inequality gives us
$$
\Vert \Phi'\tilde{r}^* - J^\mu \Vert = \Vert \Phi' \tilde{r}^* - \Phi' r^* + \Phi' r^* - J^\mu \Vert \leq \Vert \Phi' \tilde{r}^* - \Phi' r^* \Vert + \Vert \Phi' r^* - J^\mu \Vert.
$$
This indicates that the distance between $\Phi' \tilde{r}^*$ and $J^\mu$ is bounded by the distance between the falsified approximation $\Phi' \tilde{r}^*$ of $J^\mu$ and the true approximation $\Phi' r^*$ of $J^\mu$ plus the approximation error $\Vert \Phi' r^* - J^\mu \Vert$.
Moreover, we know from Theorem 1 in \cite{tsitsiklis1997analysis} that 
\begin{equation}\label{TDApproErrorUnderTrueCost}
\Vert \Phi' r^* - J^ \mu \Vert_D \leq \frac{1-\lambda \alpha}{\sqrt{(1-\alpha)(1+\alpha -2\lambda \alpha)}} \min_r \Vert \Phi' r - J^\mu \Vert_D.
\end{equation}
From Lemma 6.8, we know there exists $\Pi=\Phi (\Phi' D \Phi)^{-1} \Phi' D$ such that for every vector $J$, we have $\Vert \Pi J - J \Vert_D = \min_r \Vert \Phi r - J \Vert_D$. Hence, we arrive at 
$$
\Vert \Phi' \tilde{r}^* - J^ \mu \Vert_D \leq \Vert \Phi'\tilde{r}^* - \Phi' r^* \Vert_D + \frac{1-\lambda \alpha}{\sqrt{(1-\alpha)(1+\alpha -2\lambda \alpha)}} \Vert (\Pi-I) J^\mu\Vert_D.
$$\qed
\end{proof}

Note that $\Vert (\Pi - I) J^\mu \Vert$ can be further bounded by $\Vert (\Pi -I) J^\mu \Vert = \Vert J^\mu \Vert - \Vert \Pi J^\mu \Vert \leq \Vert J^\mu \Vert$. But (\ref{ApproErrorBound}) provides a tighter bound. From (\ref{TDApproErrorUnderTrueCost}), we know that for the case $\lambda=1$, the TD($\lambda$) algorithm under the true costs (\ref{ApproximateTDupdate}) and (\ref{ApproximateEligibility}) gives the best approximation of $J^\mu$, i.e., $\Phi' r^* = \Pi J^\mu$ and $\Vert \Phi' r^* - J^ \mu \Vert_D = \min_r \Vert \Phi' r - J^\mu \Vert_D$. As $\lambda$ decreases, $(1-\alpha\lambda)/ \sqrt{(1-\alpha)(1+\alpha -2\lambda \alpha)}$ increases and the bound deteriorates. The worst bound, namely, $\Vert (\Pi-I)J^\mu \Vert/\sqrt{1-\alpha^2}$ is obtained when $\lambda=0$. Although the result provides a bound, it suggests that in the worst-case scenario, as $\lambda$ decreases, the TD($\lambda$) under costs manipulation can suffer higher approximation error. From Lemma \ref{TDApproErrorUnderTrueCost}, we know that the manipulation of the costs has no impact on the second part of the bound which is $(1-\lambda \alpha)\Vert (\Pi-I) J^\mu\Vert_D/ \sqrt{(1-\alpha)(1+\alpha -2\lambda \alpha)}$. In the following theorem, we analyze how much $\Vert \Phi' (\tilde{r}^* -r^*)\Vert_D$ will change under a bounded manipulation of the costs.

\begin{theorem}\label{TDAttackBound}
Suppose the manipulation of the costs is given by $\eta\in \mathbb{R}^n$ where $\eta(i)= \tilde{g}(i,\mu(i)) - g(i,\mu(i)) $ for $i\in S$. Then, the distance between the manipulated TD($\lambda$) estimate and the true TD($\lambda$) estimate of $J^\mu$ satisfies
$$
\Vert \Phi' \tilde{r}^* - \Phi' r^* \Vert_D  \leq \frac{1}{1-\alpha} \Vert \eta \Vert_D, 
$$
and the approximation error of $\Phi' \tilde{r}^*$ satisfies
\begin{equation}\label{FalTDBound}
\Vert \Phi' \tilde{r}^* - J^ \mu \Vert_D \leq \Vert \frac{1}{1-\alpha} \Vert \eta \Vert_D + \frac{1-\lambda \alpha}{\sqrt{(1-\alpha)(1+\alpha -2\lambda \alpha)}} \Vert (\Pi-I) J^\mu\Vert_D.
\end{equation}
\end{theorem}
\begin{proof}
Note that $\tilde{r}^*$ and $r^*$ satisfies $A\tilde{r}^* + \tilde{b}=0$ and $Ar^* + b =0$ where $b=\Phi'D\sum_{m=0}^\infty (\alpha \lambda P_\mu)^m \bar{g}$. Here, $\bar{g}$ is vector whose $i$-th component is $g(i,\mu(i))$. 

Then, we have $A(\tilde{r}^*-r^*) + \tilde{b}-b =0$ which implies (because $\Phi' D$ is of full rank)
$$
M\Phi (\tilde{r}^* -r^*) + \sum_{m=0}^\infty (\alpha \lambda P_\mu)^m (\tilde{g}-\bar{g}) =\Phi (\tilde{r}^* -r^*).
$$
Applying $\Vert \cdot\Vert_D$ on both side of the equation, we obtain
$$
\Vert M\Phi (\tilde{r}^* -r^*) \Vert_D + \Vert \sum_{m=0}^\infty (\alpha \lambda P_\mu)^m \eta \Vert_D = \Vert \Phi (\tilde{r}^* -r^*) \Vert.
$$

From Lemma 6.4 in \cite{bertsekas1996neuro}, we know that for any $J\in\mathbb{R}^n$, we have $\Vert P_\mu J \Vert_D \leq \Vert J \Vert$.  From this, it easily follows that $\Vert P^m_\mu J \Vert_D\leq \Vert J\Vert_D$.
Note that $M=(1-\lambda)\sum_{m=0}^\infty \lambda^m (\alpha P_{\mu})^{m+1}$. Using the triangle inequality, we obtain
$$
\Vert MJ\Vert_D \leq  (1-\lambda) \sum_{m=0}^\infty \lambda^m \alpha^{m+1} \Vert J\Vert_D = \frac{\alpha(1-\lambda)}{1-\alpha \lambda} \Vert J\Vert_D.
$$
Hence, we have
$$
\Vert \sum_{m=0}^\infty (\alpha \lambda P_\mu)^m \eta \Vert_D \geq  \frac{1-\alpha}{1-\alpha \lambda}\Vert \Phi (\tilde{r}^* -r^*) \Vert .
$$
Moreover, we can see that
$$
\Vert \sum_{m=0}^\infty (\alpha \lambda P_\mu)^m \eta \Vert_D \leq \Vert \sum_{m=0}^\infty (\alpha \lambda P_\mu)^m \Vert_D \Vert \eta\Vert_D \leq \sum_{m=0}^\infty (\alpha \lambda)^m \Vert \eta\Vert_D = \frac{1}{1-\alpha\lambda} \Vert \eta\Vert_D,
$$
which indicates that
$$
\frac{1}{1-\alpha \lambda} \Vert \eta\Vert_D \geq  \frac{1-\alpha}{1-\alpha \lambda}\Vert \Phi (\tilde{r}^* -r^*) \Vert_D.
$$
Thus, we have
$$
\Vert \Phi (\tilde{r}^* -r^*) \Vert_D \leq \frac{1}{1-\alpha} \Vert \eta\Vert_D.
$$
Together with Lemma \ref{BoundFalApproximationError}, we have
$$
\Vert \Phi' \tilde{r}^* - J^ \mu \Vert_D \leq  \frac{1}{1-\alpha} \Vert \eta \Vert_D + \frac{1-\lambda \alpha}{\sqrt{(1-\alpha)(1+\alpha -2\lambda \alpha)}} \Vert (\Pi-I) J^\mu\Vert_D.
$$
\qed
\end{proof}

The distance between the true TD($\lambda$) approximator $\Phi'r^*$ and the manipulated TD($\lambda$) approximator $\Phi' \tilde{r}^*$ is actually bounded by a constant term times the distance of the true cost and the manipulated cost. And the constant term $1/(1-\alpha)$ does not depend on $\lambda$. This means the choice of $\lambda$ by the agent does not affect the robustness of the TD($\lambda$) algorithm. This means in the worst-case scenario, TD($\lambda$) algorithms with different values of $\lambda$ will suffer the same loss of approximation accuracy. From (\ref{FalTDBound}), we can conclude that the approximation error of the manipulated TD($\lambda$) algorithm is bounded by the magnitude of the costs manipulation and a fixed value decided by the value of $\lambda$, the choice of basis $\Phi$ and the properties of the MDP.

\subsection{Q-Learning with Manipulated Cost Signals}

If the RL agent learns an optimal policy by $Q$-learning algorithm given in (\ref{QlearningUpdate}), then under stealthy attacks on cost, the algorithm can be written as 
\begin{equation}\label{FalsifiedQLearning}
\begin{aligned}
\tilde{Q}_{t+1}(i,u) = (1-\gamma_t)\tilde{Q}_t(i,u) + \gamma_t  \left( \tilde{g}(i,u)  + \alpha \min_{v\in U(\bar{i})} \tilde{Q}_t(\bar{i},v)  \right).
\end{aligned}
\end{equation}
Note that if the attacks are not stealthy, we need to write $\tilde{g}_t$ in lieu of $\tilde{g}(i_t,a_t)$. There are two important questions regarding the $Q$-learning algorithm with falsified cost (\ref{FalsifiedQLearning}): (1) Will the sequence of $Q_t$-factors converge? (2) Where will the sequence of $Q_t$ converge to? 

Suppose that the sequence $\tilde{Q}_t$ generated by the $Q$-learning algorithm (\ref{FalsifiedQLearning}) converges. Let $\tilde{Q}^*$ be the limit, i.e., $\tilde{Q}^* = \lim_{n\rightarrow\infty} \tilde{Q}_t$. Suppose the objective of an adversary is to induce the RL agent to learn a particular policy $\mu^\dagger$. The adversary's problem then is to design $\tilde{g}$ by applying the actions available to him based on the information he has so that the limit $Q$-factors learned from the $Q$-learning algorithm produce the policy favored by the adversary $\mu^\dagger$, i.e, $\tilde{Q}^*\in \mathcal{V}_{\mu^\dagger}$, where 
$$
\mathcal{V}_{\mu}\coloneqq \{Q\in\mathbb{R}^{n \times |A|}: \mu(i)=\arg\min_u Q(i,u), \forall i\in S \}.
$$

In $Q$-learning algorithm (\ref{FalsifiedQLearning}), to guarantee almost sure convergence, the agent usually takes tapering stepsize \cite{borkar2009stochastic} $\{\gamma_t\}$ which satisfies $0<\gamma_t\leq 1$, $t\geq 0$, and $\sum_t \gamma_t =\infty$, $\sum_t \gamma_t^2 < \infty$. Suppose in our problem, the agent takes tapering stepsize. To address the convergence issues, we have the following result. 

\begin{lemma}\label{ConvergenceFalQ}
If an adversary performs stealthy attacks with bounded $\tilde{g}(i,a,j)$ for all $i,j\in S,a\in A$, then the $Q$-learning algorithm with falsified costs converges to the fixed point of $\tilde{F}(Q)$ almost surely where the mapping $\tilde{F}:\mathbb{R}^{n \times |A|} \rightarrow \mathbb{R}^{n\times |A|}$ is defined as $\tilde{F}(Q)=[\tilde{F}_{ii}(Q)]_{i,i}$ with
$$
\tilde{F}_{iu}(Q)=\alpha \sum_j p_{ij}(u) \min_v Q(j,v) +\tilde{g}(i,u,j),
$$
and the fixed point is unique and denoted by $\tilde{Q}^*$.
\end{lemma}

The proof of Lemma 1.1 is included in \cite{huang2019deceptive}. It is not surprising that one of the conditions given in Lemma \ref{ConvergenceFalQ} is that an attacker performs stealthy attacks. The convergence can be guaranteed because the falsified cost signals are consistent over time for each state action pair. The uniqueness of $\tilde{Q}^*$ comes from the fact that if $\tilde{g}(i,u)$ is bounded for every $(i,u)\in S\times A$, $\tilde{F}$ is a contraction mapping. By Banach's fixed point theorem \cite{kreyszig1978introductory}, $\tilde{F}$ admits a unique fixed point. With this lemma, we conclude that an adversary can make the algorithm converge to a limit point by stealthily falsifying the cost signals. 

\begin{remark}
Whether an adversary aims for the convergence of the $Q$-learning algorithm (\ref{FalsifiedQLearning}) or not depends on his objective. In our setting, the adversary intends to mislead the RL agent to learn  policy $\mu^\dag$, indicating that the adversary promotes convergence and aim to have the limit point $\tilde{Q}^*$ lie in $\mathcal{V}_{\mu^\dag}$.
\end{remark}

It is interesting to analyze, from the adversary's perspective, how to falsify the cost signals so that the limit point that algorithm (\ref{FalsifiedQLearning}) converges to is favored by the adversary. In later discussions, we consider stealthy attacks where the falsified costs are consistent for the same state action pairs. Denote the true cost by matrix $g\in\mathbb{R}^{n\times |A|}$ with $[g]_{i,u}=g(i,u)$ and the falsified cost is described by a matrix $\tilde{g}\in\mathbb{R}^{n\times |A|}$ with $[\tilde{g}]_{i,u}=\tilde{g}(i,u)$. Given $\tilde{g}$, the fixed point of $\tilde{F}$ is uniquely decided; i.e., the point that the algorithm (\ref{FalsifiedQLearning}) converges to is uniquely determined. Thus, there is a mapping $\tilde{g}\mapsto\tilde{Q}^*$ implicitly described by the relation $\tilde{F}(Q)=Q$. For convenience, this mapping is denoted by $f:\mathbb{R}^{n\times |A|}\rightarrow \mathbb{R}^{n\times |A|}$.

\begin{theorem}[No Butterfly Effect]\label{LipContTheo}
Let $\tilde{Q}^*$ denote the $Q$-factor learned from algorithm (\ref{QlearningUpdate}) with falsified cost signals and $Q^*$ be the $Q$-factor learned from (\ref{FalsifiedQLearning}) with true cost signals. There exists a constant $L<1$ such that
\begin{equation}\label{LipCont}
    \Vert \tilde{Q}^* - Q^* \Vert  \leq  \frac{1}{1-L}\Vert \tilde{g}- g \Vert,
\end{equation}
and $L=\alpha$.
\end{theorem}

The proof of Theorem \ref{LipContTheo} can be found in \cite{huang2019deceptive}. In fact, taking this argument just slightly further, one can conclude that falsification on cost $g$ by a tiny perturbation does not cause significant changes in the limit point of algorithm (\ref{FalsifiedQLearning}), $\tilde{Q}^*$. This feature indicates that an adversary cannot cause a significant change in the limit $Q$-factors by just applying a small perturbation in the cost signals. 
This is a feature known as stability, which is observed in problems that possess contraction mapping properties. Also, Theorem \ref{LipContTheo} indicates that the mapping $\tilde{g}\mapsto \tilde{Q}^*$ is continuous, and to be more specific, it is uniformly Lipschitz continuous with Lipschitz constant ${1}/{(1-\alpha)}$.

With Theorem \ref{LipContTheo}, we can characterize the minimum level of falsification required to change the policy from the true optimal policy $\mu^*$ to the policy $\mu^\dag$. First, note that $\mathcal{V}_\mu\subset \mathbb{R}^{n\times |A|}$ and it can also be written as 
\begin{equation}\label{VictorySet}
\mathcal{V}_\mu=\{Q\in\mathbb{R}^{n\times |A|}: Q(i,\mu(i))< Q(i,u), \forall i\in S, \forall u\neq \mu(i) \}.
\end{equation}
Second, for any two different policies $\mu_1$ and $\mu_2$, $\mathcal{V}_{\mu_1}\cap \mathcal{V}_{\mu_2}=\varnothing$. Lemma \ref{LemmaForSet} presents several important properties regarding the set $\mathcal{V}_\mu$.

\begin{lemma}\label{LemmaForSet}
\begin{enumerate}[(a)]
    \item For any given policy $\mu$, $\mathcal{V}_\mu$ is a convex set.
    \item For any two different policies $\mu_1$ and $\mu_2$, $\mathcal{V}_{\mu_1}\cap \mathcal{V}_{\mu_2}=\varnothing$.
     \item The distance between any two different policies $\mu_1$ and $\mu_2$ defined as $D(\mu_1,\mu_2)\coloneqq\inf_{Q_1\in \mathcal{V}_{\mu_1}, Q_2\in\mathcal{V}_{\mu_2}} \Vert Q_1 - Q_2 \Vert$ is zero.
\end{enumerate}
\end{lemma}
\begin{proof}
(a) Suppose $Q_1,Q_2 \in \mathcal{V}_{\mu}$. We show for every $\lambda \in [0,1]$, $\lambda Q_1 + (1-\lambda) Q_2\in \mathcal{V}_\mu$. This is true because 
$ Q_1(i,\mu(i))< Q_1(i,u)$ and $ Q_2(i,\mu(i))< Q_2(i,u)$ imply 
$$
\lambda Q_1(i,\mu(i)) + (1-\lambda) Q_2(i,\mu(i)) <   \lambda Q_1(i,u) + (1-\lambda) Q_2(i,u),
$$
for all $i\in S, u\neq \mu(i)$.

(b) Suppose $\mu_1 \neq \mu_2$ and $\mathcal{V}_{\mu_1} \cap \mathcal{V}_{\mu_2}$ is not empty. Then there exists $Q\in \mathcal{V}_{\mu_1} \cap \mathcal{V}_{\mu_2}$. Since $\mu_1 = \mu_2$, there exists $i$ such that $\mu_1(i)\neq \mu_2(i)$. Let $u=\mu_2$. Since $Q\in \mathcal{V}_{\mu_1}$, we have $Q(i,\mu_1(i))<Q(i,\mu_2(i))$. Hence, $Q\notin \mathcal{V}_{\mu_2}$, which is a contradiction. Thus, $\mathcal{V}_{\mu_1}\cap \mathcal{V}_{\mu_2}=\varnothing$. 

(c) Suppose $\mu_1 \neq \mu_2$. Construct $Q_1$ as a matrix whose entries are all one except $Q(i,\mu_1(i))= 1- \epsilon/2$ for every $i\in S$ where $\epsilon>0$. Similarly, construct $Q_2$ as a matrix whose entries are all one except $Q(i,\mu_2(i))= 1- \epsilon/2$ for every $i\in S$. It is easy to see that $Q_1\in \mathcal{V}_{\mu_1}$ and $Q_2\in\mathcal{V}_{\mu_2}$. Then $\inf_{Q_1\in \mathcal{V}_{\mu_1}, Q_2\in\mathcal{V}_{\mu_2}} \Vert Q_1 - Q_2 \Vert_\infty \leq\Vert Q_1 - Q_2 \Vert_\infty =\epsilon$. Since $\epsilon$ can be arbitrarily small, $\inf_{Q_1\in \mathcal{V}_{\mu_1}, Q_2\in\mathcal{V}_{\mu_2}} \Vert Q_1 - Q_2 \Vert_\infty=0$. Since norms are equivalent in finite-dimensional space (see Sec. 2.4 in \cite{kreyszig1978introductory}), we have $D(\mu_1,\mu_2)=\inf_{Q_1\in \mathcal{V}_{\mu_1}, Q_2\in\mathcal{V}_{\mu_2}} \Vert Q_1 - Q_2 \Vert=0$.\qed
\end{proof}

Suppose the true optimal policy $\mu^*$ and the adversary desired policy $\mu^\dag$ are different; otherwise, the optimal policy $\mu^*$ is what the adversary desires, there is no incentive for the adversary to attack. According to Lemma \ref{LemmaForSet}, $D(\mu^*,\mu^\dag)$ is always zero. This counterintuitive result states that a small change in the $Q$-value may result in any possible change of policy learned by the  agent from the $Q$-learning algorithm (\ref{FalsifiedQLearning}). Compared with Theorem \ref{LipContTheo} which is a negative result to the adversary, this result is in favor of the adversary.

Define the point $Q^*$ to set 
$\mathcal{V}_{\mu^\dag}$ distance by 
$
D_{Q^*}(\mu^\dag)\coloneqq \inf_{Q\in\mathcal{V}_{\mu^\dag}} \Vert Q-Q^* \Vert.
$
Thus, if $\tilde{Q}^*\in\mathcal{V}_{\mu^\dag}$, we have 
\begin{equation}\label{RobustRegionIne}
0= D(\mu^*,\mu^\dagger)\leq D_{Q^*}(\mu^\dag)\leq \Vert \tilde{Q}^*-Q^* \Vert  \leq\frac{1}{1-\alpha}\Vert \tilde{g} - g \Vert,
\end{equation}
where the first inequality comes from the fact that $Q^*\in \mathcal{V}_{\mu^*}$ and the second inequality is due to $\tilde{Q}^*\in V_{\mu^\dag}$. The inequalities give us the following theorem. 

\begin{theorem}[Robust Region]
To make the agent learn the policy $\mu^\dag$, the adversary has to manipulate the cost such that $\tilde{g}$ lies outside the ball $\mathcal{B}(g;(1-\alpha)D_{Q^*}(\mu^\dag))$.
\end{theorem}

The \textit{robust region} for the true cost $g$ to the adversary's targeted policy $\mu^\dag$ is given by $\mathcal{B}(g;(1-\alpha)D_{Q^*}(\mu^\dag) )$ which is an open ball with center $c$ and radius $(1-\alpha)D_{Q^*}(\mu^\dag)$. That means the attacks on the cost needs to be `powerful' enough to drive the falsified cost $\tilde{g}$ outside the ball $\mathcal{B}(g;(1-\alpha)D_{Q^*}(\mu^\dag) )$ to make the RL agent learn the policy $\mu^\dag$. If the falsified cost $\tilde{g}$ is within the ball, the RL agent can never learn the adversary's targeted policy $\mu^\dag$. The ball $\mathcal{B}(g;(1-\alpha)D_{Q^*}(\mu^\dag))$ depends only on the true cost $g$ and the adversary desired policy $\mu^\dag$ (Once the MDP is given, $Q^*$ is uniquely determined by $g$). Thus, we refer this ball as the robust region of the true cost $g$ to the adversarial policy $\mu^\dag$. As we have mentioned, if the actions available to the adversary only allows him to perform bounded falsification on cost signals and the bound is smaller than the radius of the robust region, then the adversary can never mislead the agent to learn policy $\mu^\dag$.

\begin{remark}
First, in discussions above, the adversary policy $\mu^\dag$ can be any possible polices and the discussion remains valid for any possible policies. Second, set $\mathcal{V}_{\mu}$ of $Q$-values is not just a convex set but also an open set. We thus can see that $D_{Q^*}(\mu^\dag)>0$ for any $\mu^\dag \neq \mu^*$ and the second inequality in (\ref{RobustRegionIne}) can be replaced by a strict inequality. Third, the agent can estimate his own robustness to falsification if he can know the adversary desired policy $\mu^\dag$. For attackers who have access to true cost signals and the system model, the attacker can compute the robust region of the true cost to his desired policy $\mu^\dag$ to evaluate whether the objective is feasible or not. When it is not feasible, the attacker can consider changing his objectives, e.g., selecting other favored policies that have a smaller robust region.
\end{remark}

We have discussed how falsification affects the change of $Q$-factors learned by the agent in a distance sense. The problem now is to study how to falsify the true cost in a right direction so that the resulted $Q$-factors fall into the favored region of an adversary. One difficulty of analyzing this problem comes from the fact that the mapping $\tilde{g}\mapsto \tilde{Q}^*$ is not explicit known. The relation between $\tilde{g}$ and $\tilde{g}^*$ is governed by the $Q$-learning algorithm (\ref{FalsifiedQLearning}). Another difficulty is that due to the fact that both $\tilde{g}$ and $\tilde{Q}^*$ lies in the space of $\mathbb{R}^{n\times |A|}$, we need to resort to Fr\'echet derivative or G\^ateaux derivative \cite{cheney2013analysis} (if they exist) to characterize how a small change of $\tilde{g}$ results in a change in $\tilde{Q}^*$.

From Lemma \ref{ConvergenceFalQ} and Theorem \ref{LipContTheo}, we know that $Q$-learning algorithm converges to the unique fixed point of $\tilde{F}$ and that $f:\tilde{g}\mapsto \tilde{Q}^*$ is uniformly Lipschitz continuous. Also, it is easy to see that the inverse of $f$, denoted by $f^{-1}$, exists since given $\tilde{Q}^*$, $\tilde{g}$ is uniquely decided by the relation $\tilde{F}(Q)=Q$. Furthermore, by the relation $\tilde{F}(Q)=Q$, we know $f$ is both injective and surjective and hence a bijection which can be simply shown by arguing that given different $\tilde{g}$, the solution of $\tilde{F}(Q)=Q$ must be different. This fact informs that there is a one-to-one, onto correspondence between $\tilde{g}$ and $\tilde{Q}^*$. One should note that the mapping $f:\mathbb{R}^{n\times |A|}\rightarrow \mathbb{R}^{n\times |A|}$ is not uniformly Fr\'echet differentiable on $\mathbb{R}^{n\times |A|}$ due to the $\min$ operator inside the relation $\tilde{F}(Q)=Q$. However, for any policy $\mu$, $f$ is Fr\'echet differentiable on $f^{-1}(\mathcal{V}_\mu)$ which is an open set and connected due to the fact that $\mathcal{V}_\mu$ is open and connected (every convex set is connected) and $f$ is continuous. In the next lemma, we show that $f$ is Fr\'echet differentiable on $f^{-1}(\mathcal{V}_w)$ and the derivative is constant.

\begin{lemma}\label{Diffentialf}
The map $f:\mathbb{R}^{n\times |A|}\rightarrow \mathbb{R}^{n\times |A|}$ is Fr\'echet differentiable on $f^{-1}(\mathcal{V}_\mu)$ for any policy $\mu$ and the Fr\'echet derivative of $f$ at any point $\tilde{g}\in \mathcal{V}_\mu$, denoted by $f'(\tilde{g})$, is a linear bounded map $G:\mathbb{R}^{n\times |A|} \rightarrow \mathbb{R}^{n \times |A|}$ that does not depend on $\tilde{g}$, and 
$Gh$ is given as
\begin{equation}\label{FrechetDerivative}
[Gh]_{i,u}= \alpha P_{iu}^T(I-\alpha P_\mu)^{-1}h_\mu + h(i,u) 
\end{equation}
for every $i\in S,u\in A$, where $P_{iu} = (p_{i1}(u),\cdots,p_{in}(u))$, $P_\mu \coloneqq [P_\mu]_{i,j}=p_{ij}(\mu(i))$.
\end{lemma}

The proof of Lemma \ref{Diffentialf} is provided in \cite{huang2019deceptive}. We can see that $f$ is Fr\'echet differentiable on $f^{-1}(\mathcal{V}_\mu)$ and the derivative is constant, i.e., $f'(\tilde{g})=G$ for any $\tilde{g}\in f^{-1}(\mathcal{V}_\mu)$. Note that $G$ lies in the space of all linear mappings that maps $\mathbb{R}^{n\times |A|}$ to itself and $G$ is determined only by the discount factor $\alpha$ and the transition kernel $P$ of the MDP problem. The region where the differentiability may fail is  $f^{-1}(\mathbb{R}^{n\times |A|}\backslash (\cup_{\mu} \mathcal{V}_\mu))$, where $\mathbb{R}^{n \times |A|}\backslash (\cup_{\mu} \mathcal{V}_\mu)$ is the set $\{Q: \exists i, \exists u\neq u', Q(i,u)=Q(i,u')=\min_v Q(i,v) \}$. This set contains the places where a change of policy happens, i.e., $Q(i,u)$ and $Q(i,u')$ are both the lowest value among the $i$th row of $Q$. Also, due to the fact that $f$ is Lipschitz, by Rademacher's theorem, $f$ is differentiable almost everywhere (w.r.t. the Lebesgue measure).

\begin{remark}
One can view $f$ as a `piece-wise linear function' in the norm vector space $\mathbb{R}^{n\times |A|}$. Actually, if the adversary can only falsify the cost at one state-control pair, say $(i,u)$, while costs at other pairs are fixed, then for every $j\in S,v\in A$, the function $\tilde{g}(i,u)\mapsto [\tilde{Q}^*]_{j,v}$ is a piece-wise linear function.
\end{remark}

Given any $g \in f^{-1}(\mathcal{V}_\mu)$, if an adversary falsifies the cost $g$ by injecting value $h$, i.e., $\tilde{g}=g+h$, the adversary can see how the falsification causes a change in $Q$-values. To be more specific, if $Q^*$ is the $Q$-values learned from cost $g$ by $Q$-learning algorithm (\ref{QlearningUpdate}), after the falsification $\tilde{g}$, the $Q$-value learned from $Q$-learning algorithm (\ref{FalsifiedQLearning}) becomes $\tilde{Q}^*=Q^*+Gh$ if $\tilde{g}\in f^{-1}(\mathcal{V}_\mu)$. Then, an adversary who knows the system model can utilize (\ref{FrechetDerivative}) to find a way of falsification $h$ such that $\tilde{Q}^*$ can be driven to approach a desired set $\mathcal{V}_{\mu^\dag}$. One difficulty is to see whether $\tilde{g}\in f^{-1}(\mathcal{V}_\mu)$ because the set $f^{-1}(\mathcal{V}_\mu)$ is now implicit expressed. Thus, we resort to the following theorem.

\begin{theorem}\label{iffTheorem}
Let $\tilde{Q}^*\in\mathbb{R}^{n\times |A|}$ be the $Q$-values learned from the $Q$-learning algorithm (\ref{FalsifiedQLearning}) with the falsified cost $\tilde{g}\in\mathbb{R}^{n\times |A|}$. Then $\tilde{Q}^* \in \mathcal{V}_{\mu^\dag}$ if and only if the falsified cost signals $\tilde{g}$ designed by the adversary satisfy the following conditions
\begin{equation}\label{FalCostConds}
\tilde{g}(i,u)> (\mathbf{1}_i - \alpha P_{iu})^T(I-\alpha P_{\mu^\dag})^{-1}\tilde{g}_{\mu^\dagger}.
\end{equation}
for all $i\in S$, $u\in A \backslash \{\mu^\dag(i)\}$, where $\mathbf{1}_i\in\mathbb{R}^n$ a vector with $n$ components whose $i$th component is $1$ and other components are $0$.
\end{theorem}

With the results in Theorem \ref{iffTheorem}, we can characterize the set $f^{-1}(\mathcal{V}_\mu)$. Elements in $f^{-1}(\mathcal{V}_\mu)$ have to satisfy the conditions given in (\ref{FalCostConds}). Also, Theorem \ref{iffTheorem} indicates that if an adversary intends to mislead the agent to learn policy $\mu^\dag$, the falsified cost $\tilde{g}$ has to satisfy the conditions specified in (\ref{FalCostConds}).

If an adversary can only falsify at certain states $\tilde{S}\subset S$, the adversary may not be able to manipulate the agent to learn $\mu^\dag$. Next, we study under what conditions the adversary can make the agent learn $\mu^\dag$ by only manipulating the costs on $\tilde{S}$. Without loss of generality, suppose that the adversary can only falsify the cost at a subset of states $\tilde{S}=\{1,2,...,|\tilde{S}|\}$. We rewrite the conditions given in (\ref{FalCostConds}) into a more compact form:
\begin{equation}\label{FalCostCondsCompact}
    \tilde{g}_u \geq (I-\alpha P_u)(I-\alpha P_{\mu^\dag})^{-1} \tilde{g}_{\mu^\dag}, \forall\ u\in A,
\end{equation}
where $\tilde{g}_u = (\tilde{g}(1,u),\dots,\tilde{g}(n,u))$, $\tilde{g}_{\mu^\dag} = \left(\tilde{g}(1,\mu^\dag(1)), \tilde{g}(2,\mu^\dag(2)),\dots, \tilde{g}(n,\mu^\dag(n))\right)$ and $P_u = [P_u]_{i,j}=p_{ij}(u)$. The equality only holds for one component of the vector, i.e., the $i$-th component satisfying $\mu(i)=u$. Partition the vector $\tilde{g}_u$ and $\tilde{g}_{\mu^\dag}$ in (\ref{FalCostCondsCompact}) into two parts, the part where the adversary can falsify the cost denoted by $\tilde{g}^{fal}_u,\tilde{g}^{fal}_{\mu^\dag}\in\mathbb{R}^{|\tilde{S}|}$ and the part where the adversary cannot falsify $g_u^{true},g_{\mu^\dag}^{true}\in \mathbb{R}^{n-|\tilde{S}|}$. Then (\ref{FalCostCondsCompact}) can be written as 
\begin{equation}\label{PartionedCostCon}
\begin{bmatrix}
\tilde{g}^{fal}_u\\ \hdashline[2pt/2pt]
g_u^{true}
\end{bmatrix}
 \geq \left[
    \begin{array}{c;{2pt/2pt}c}
        R_u  & Y_u \\ \hdashline[2pt/2pt]
        M_u & N_u
    \end{array}
\right]
\begin{bmatrix}
\tilde{g}^{fal}_{\mu^\dag}\\ \hdashline[2pt/2pt]
g_{\mu^\dag}^{true}
\end{bmatrix},\ \forall\ u\in A,
\end{equation}
where 
$$
\left[
    \begin{array}{c;{2pt/2pt}c}
        R_u  & Y_u \\ \hdashline[2pt/2pt]
        M_u & N_u
    \end{array}
\right]\coloneqq
(I-\alpha P_u) (I-\alpha P_{\mu^\dag})^{-1},\ \ \forall\ u\in A
$$
and $R_u\in\mathbb{R}^{\tilde{S}\times \tilde{S}}, Y_u\in\mathbb{R}^{|\tilde{S}|\times (n-|\tilde{S}|)},M_u\in\mathbb{R}^{(n-|\tilde{S}|)\times |\tilde{S}|}, N_u\in\mathbb{R}^{(n-|\tilde{S}|)\times (n-|\tilde{S}|)}$.

If the adversary aims to mislead the agent to learn $\mu^\dag$, the adversary needs to design $\tilde{g}^{fal}_u,u\in A$ such that the conditions in (\ref{PartionedCostCon}) hold. The following results state that under some conditions on the transition probability, no matter what the true costs are, the adversary can find proper $\tilde{g}^{fal}_u,u\in A$ such that conditions (\ref{PartionedCostCon}) are satisfied. For $i\in S\backslash \tilde{S}$, if $\mu(i)=u$, we remove the rows of $M_u$ that correspond to the state $i\in S\backslash \tilde{S}$. Denote the new matrix after the row removals by $\bar{M}_u$.

\begin{proposition}\label{PartialStatesAttacks}
Define $H\coloneqq [\bar{M}_{u_1}'\ \bar{M}_{u_2}'\ \cdots\ \bar{M}_{u_{|A|}}']' \in \mathbb{R}^{(|A|(n-|\tilde{S}|)-|\tilde{S}|)\times |\tilde{S}|}$. If there exists $x\in\mathbb{R}^{|\tilde{S}|}$ such that $Hx<0$, i.e.,  the column space of $H$ intersects the negative orthant of $\mathbb{R}^{|A|(n-|\tilde{S}|)-|\tilde{S}|}$,  then for any true cost, the adversary can find $\tilde{g}^{fal}_u,u\in {A}$ such that conditions (\ref{PartionedCostCon}) hold.
\end{proposition}

The proof can be found in \cite{huang2019deceptive}. Note that $H$ only depends on the transition probability and the discount factor, if an adversary who knows the system model can only falsify cost signals at states denoted by $\tilde{S}$, an adversary can check if the range space of $H$ intersects with the negative orthant of $\mathbb{R}^{|A|(n-|\tilde{S}|)}$ or not. If it does, the adversary can mislead the agent to learn $\mu^\dag$ by falsifying costs at a subset of state space no matter what the true cost is. 

\begin{remark}
To check whether the condition on $H$ is true or not, one has to resort to Gordan's theorem \cite{broyden2001theorems}:
Either $Hx<0$ has a solution $x$, or $H^T y =0$ has a nonzero solution $y$ with $y\geq 0$. The adversary can use linear/convex programming software to check if this is the case. For example, by solving
\begin{equation}\label{GordanMin}
\begin{aligned}
\min_{y\in\mathbb{R}^{|A|(n-|\tilde{S}|)}}\ \ \ &\Vert H^T y \Vert\ \  s.t.\ \ \ \ \Vert y \Vert=1,\ y\geq 0,
\end{aligned}
\end{equation}
the adversary can know whether the condition about $H$ given in Proposition \ref{PartialStatesAttacks} is true or not. If the minimum of (\ref{GordanMin}) is positive, there exists $x$ such that $Hx<0$. The adversary can select $\tilde{g}_{\mu^\dag}^{fal}=\lambda x$ and choose a sufficiently large $\lambda$ to make sure that conditions (\ref{PartionedCostCon}) hold, which means an adversary can make the agent learn the policy $\mu^\dag$ by falsifying costs at a subset of state space no matter what the true costs are.

\end{remark}

\section{Case Study}\label{CaseStudy}

Here, we consider TD learning on random walk. Given a policy $\mu$, an MDP can be considered as a Markov cost process, or MCP. In this MCP, we have $n=20$ states. The transition digram of the MCP is given in Fig. \ref{fig:RandomWalkDiagram}. At state $i=i_{k}$, $k=2,3,\dots,n-1$, the process proceed either left to $i_{k-1}$ or right to $i_{k+1}$, with equal probability. The transition at states from $i_2$ to $i_{n-1}$ is similar to symmetric one-dimensional random walk.  At state $i_{1}$, the process proceed to state $i_2$ with probability $1/2$ or stays at the same state with equal probability. At state $i_n$, the probabilities of transition to $i_{n-1}$ and staying at $i_n$ are both $\frac{1}{2}$. That is, we have $p_{i_ki_{k+1}}(\mu(i_k)) = p_{i_k i_{k-1}}(\mu(i_k)) =\frac{1}{2}$ for $k=2,3,\dots,n-1$, $p_{i_1i_1}=p_{i_1 i_2}=\frac{1}{2}$ and $p_{i_ni_n}=p_{i_ni_{n-1}}=\frac{1}{2}$. The cost at state $i_k$ is set to be $k$ if $k\leq 10$ and $21-k$ if $k>10$. That is
$$
g(i_k,\mu(i_k)) = \begin{cases}
k\ \ \ &\textrm{if }k\leq 10\\
21-k\ \ \ &\textrm{else}
\end{cases}.
$$
We consider the discount factor $\alpha=0.9$. The task here is to use approximate TD($\lambda$) learning algorithm to esitimate and approximate the cost-to-go function $J^\mu$ of this MCP.
We consider a linear parametrization of the form 
\begin{equation}\label{TDCaseArchitecture}
\tilde{J}(i,r)= r(3)i^2 + r(2) i + r(1),
\end{equation}
and $r=(r(1),r(2),r(3))\in\mathbb{R}^3$. Suppose the learning agent updates $r_t$ based on TD($\lambda$) learning algorithm (\ref{ApproximateTDupdate}) and (\ref{ApproximateEligibility}) and tries to find an estimate of $J^\mu$. We simulate the MCP and obtain a trajectory that long enough and its associated cost signals. We need an infinite long trajectory ideally. But here, we set the length of the trajectory to be $10^{5}$. We run respectively TD($1$) and TD($0$) on the same simulated trajectory based on rules given in (\ref{ApproximateTDupdate}) and (\ref{ApproximateEligibility}). The black line indicates the cost-to-go function of the MCP. The blue markers are the approximations of the cost-to-go function obtained by following the TD($\lambda$) algorithm (\ref{ApproximateTDupdate}) and (\ref{ApproximateEligibility}) with $\lambda=1$ and $\lambda=0$. We can see that $\tilde{J}_{TD(1)}$ and $\tilde{J}_{TD(0)}$ is a quadratic function of $i$ as we set in (\ref{TDCaseArchitecture}). Both $\tilde{J}_{TD(1)}$ and $\tilde{J}_{TD(0)}$ can serve a fairly good approximation of $J^\mu$ as we can see. The dimension of the parameters we need to update goes from $n=20$ in the TD($\lambda$) algorithm (\ref{TDLearningUpdates}) to $K=3$ in the approximation counterpart (\ref{ApproximateTDupdate}) which is more efficient computationally.

\begin{figure}[H]
    \centering
    \includegraphics[width=0.9\linewidth]{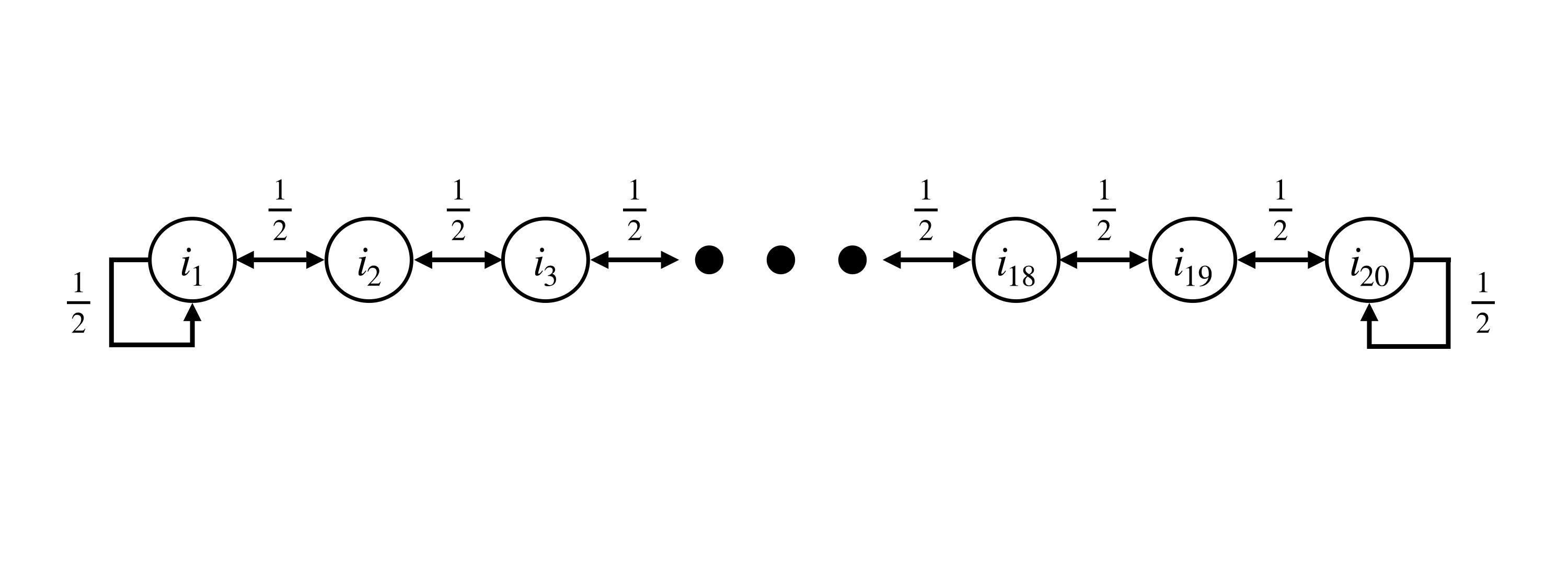}
    \caption{The diagram of the MCP task with $20$ states denoted by $i_1,\dots,i_{20}$.}
    \label{fig:RandomWalkDiagram}
\end{figure}

\begin{figure}
    \centering
    \includegraphics[width=0.9\linewidth]{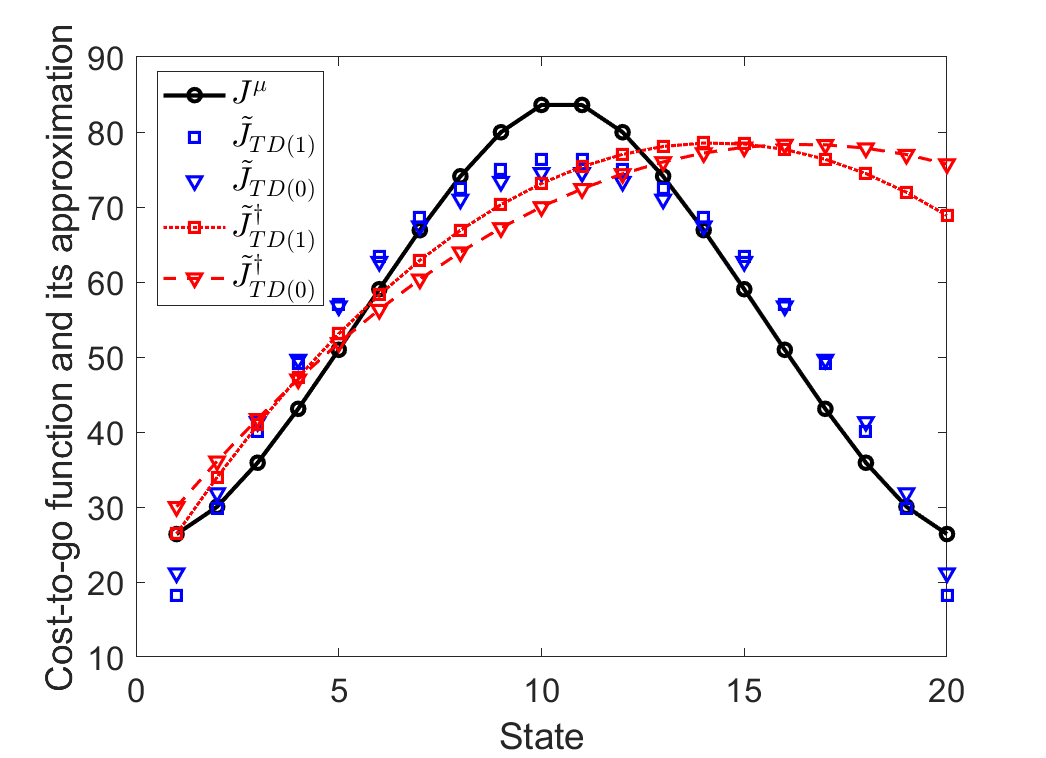}
    \caption{The cost-to-go function of a given policy, denoted by $J^\mu$; the approximations of the cost-to-go function under true cost signals which is marked blue, denoted by $\tilde{J}$; the approximations of the cost-to-go function under manipulated cost signals which is marked red, denoted by $\tilde{J}^\dag$; The subscript TD($1$) and TD($0$) denote the TD parameter $\lambda$ is set to be $1$ and $0$ respectively.}
    \label{fig:TDComparison}
\end{figure}

Suppose the adversary aims to deteriorate the TD($\lambda$) algorithm by stealthily manipulating the cost signals. Suppose the adversary can only manipulate the cost signals at state $i_{20}$ and the manipulated cost is $\tilde{g}(i_{20},\mu(i_{20}))=20$. We can see from Fig. \ref{fig:TDComparison} that the TD($\lambda$) learning under manipulated cost signals fails to provide accurate approximation of $J^\mu$. Although only the cost signal at one state is manipulated, the approximation of cost-to-go function at other states can also be largely deviated from the accurate value.

\begin{figure}
    \centering
    \includegraphics[width=0.9\linewidth]{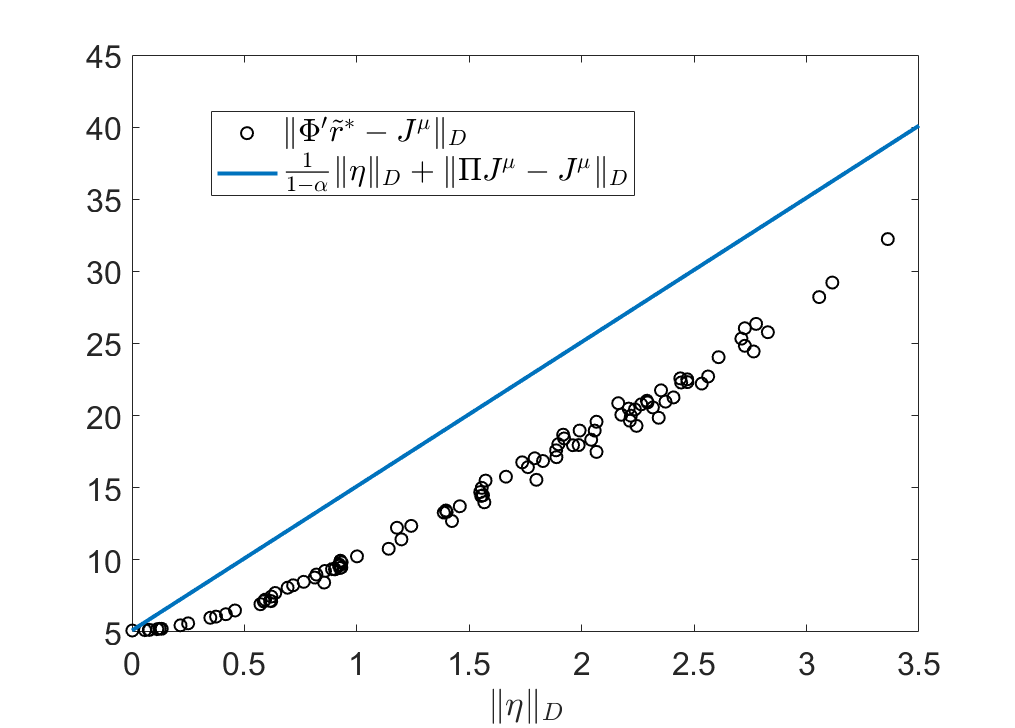}
    \caption{The approximated TD($1$) learned from $100$ random falsifications of the costs. For each falsification, we plot the distance of the falsification and the approximation error of its associated approximate of $J^\mu$. the blue line is a demonstration of the bound developed in (\ref{FalTDBound}). }
    \label{fig:BoundIllustration}
\end{figure}

We generate $100$ random falsifications denoted by $\eta$. Note that $\eta = \tilde{g} -g$. For each falsification, we plot the norm $\Vert \eta \Vert_D$ of $\eta$ and its associated approximation error $\Vert \Phi' \tilde{r}^* - J^\mu \Vert_D$ and the pair is marked by black circle in Fig. \ref{fig:BoundIllustration}. Note that $\tilde{r}^*$ is the parameter learned by the agent using TD($1$) learning algorithm (\ref{FalfisiedApproTDUpdate}) under the falsified costs. The blue line describes the map $\Vert \eta \Vert_D \mapsto \frac{1}{1-\alpha}\Vert \eta \Vert_D + \Vert \Pi J^\mu - J^\mu\Vert_D$. The results in Fig. \ref{fig:BoundIllustration} collaborate the results we proved in Theorem \ref{TDAttackBound}. This demonstrate how the falsification on cost signals in only one state can affect the approximated value function on every single state and how the resulted learning error is bounded.

For case study of $Q$-learning with falsified costs, one can refer to \cite{huang2019deceptive}.

\section{Conclusion}\label{conc}
In this chapter, we have discussed the potential threats in RL and a general framework has been introduced to study RL under deceptive falsifications of cost signals where a number of attack models have been presented. We have provided theoretical underpinnings for understanding the fundamental limits and performance bounds on the attack and the defense in RL systems. We have shown that in TD($\lambda$) the approximation learned from the manipulated costs has an approximation error bounded by a constant times the magnitude of the attack. The effect of the adversarial attacks does not depend on the choice of $\lambda$. In $Q$-learning, we have characterized a robust region within which the adversarial attacks cannot achieve its objective. The robust region of the cost can be utilized by both offensive and defensive sides. An RL agent can leverage the robust region to evaluate the robustness to malicious falsifications. An adversary can leverage it to assess the attainability of his attack objectives. Conditions given in Theorem \ref{iffTheorem} provide a fundamental understanding of the possible strategic adversarial behavior of the adversary. Theorem \ref{PartialStatesAttacks} helps understand the attainability of an adversary's objective.
Future work would focus on investigating a particular attack model and develop countermeasures to the attacks on cost signals.

\backmatter

\bibliographystyle{plainnat}
\bibliography{wiley}%

\nomenclature{$S$}{State space.}
\nomenclature{$A$}{Control space.}
\nomenclature{$g$}{Cost function(matrix).}
\nomenclature{$\tilde{g}$}{Compromised cost function(matrix).}

\nomenclature{$P$}{Transition Kernel.}

\nomenclature{$\alpha$}{Discounted factor.}
\nomenclature{$U(i)$}{Control set of state $i$.}
\nomenclature{$p_{ij}(u)$}{Transition probability from state $i$ to state $j$ under control $u$.}

\nomenclature{$\mu_k$}{Control policy at time $k$.}

\nomenclature{$i_k$}{State at time $k$.}

\nomenclature{$J^*$}{Value function of the MDP.}

\nomenclature{$\mu^*$}{Optimal policy.}

\nomenclature{$\tilde{J}$}{Apprxoimated Value function of the MDP.}

\nomenclature{$r$}{Approximation parameters.}

\nomenclature{$d_t$}{Temporal difference at time $t$.}

\nomenclature{$\gamma_t(i)$}{State-dependent non-nagative stepsize at time $t$.}

\nomenclature{$z_t(i)$}{State-dependent eligibility coefficient at state $i$.}

\nomenclature{$\phi$}{Basis functions of the approximation.}

\nomenclature{$\Phi$}{Basis functions in compact form.}

\nomenclature{$\eta_t$}{Eligibility vector at time $t$.}

\nomenclature{$D$}{Diagonal entries of steady-state probabilities.}

\nomenclature{$d(i)$}{steady-state probability at state $i$.}

\nomenclature{$P_\mu$}{Transition matrix given policy $\mu$.}

\nomenclature{$Q(i,u)$}{Q-factor of a state-control pair $(i,u)$.}

\nomenclature{$Q^*$}{Optimal Q-factors.}

\nomenclature{$\tilde{r}$}{Approximation parameters under compromised cost signals.}

\nomenclature{$\tilde{r}^*$}{Optimal approximation parameters under compromised cost signals.}

\nomenclature{$\Pi$}{Projection matrix.}

\nomenclature{$\tilde{\eta}$}{Difference between cost signals and compromised cost signals.}

\nomenclature{$\tilde{Q}_t$}{Q-factors at time $t$ under compromised cost signals.}

\nomenclature{$\tilde{Q}^*$}{Optimal Q-factors under compromised cost signals.}

\nomenclature{$\mu^\dag$}{A policy desired by the adversary.}

\nomenclature{$\mathcal{V}_\mu$}{A set of Q-factors that produces policy $\mu$.}

\nomenclature{$\tilde{F}(Q)$}{Right-hand side of Bellman's equation under compromised cost signals.}

\nomenclature{$f$}{A mapping from cost function to Q-factors.}

\nomenclature{$f^{-1}$}{The inverse of mapping $f$.}

\nomenclature{$L$}{A mapping from cost function to Q-factors.}

\nomenclature{$g_{\mu}$}{A vector of costs under policy $\mu$.}

\nomenclature{$P_{iu}$}{Transition probabilities at state $i$ under control $u$.}

\printnomenclature


\printindex

\end{document}